\documentclass{article} % For LaTeX2e
\usepackage{bm}
\usepackage{nips14submit_e,times}
\usepackage{url}
\usepackage{amssymb}
\usepackage{amsthm}
\usepackage[numbers]{natbib}
\usepackage{enumitem}
\usepackage{amsmath}  % required if you use \overset
\usepackage{array}
\newtheorem{proposition}{Proposition}

\newtheorem{theorem}{Theorem}

\usepackage{algorithm}
\usepackage{algorithmic}

\usepackage{graphicx}
\title{Generative Adversarial Nets}

\author{
Ian J. Goodfellow,
\ \  Jean Pouget-Abadie\thanks{Jean Pouget-Abadie is visiting Universit\'e de Montr\'eal from Ecole Polytechnique.
},\ \  Mehdi Mirza,\ \  Bing Xu,\ \ David Warde-Farley,\\
{\bf Sherjil Ozair\thanks{Sherjil Ozair is visiting Universit\'e de Montr\'eal from Indian Institute of Technology Delhi},\ \ Aaron Courville,\ \ Yoshua Bengio\thanks{Yoshua Bengio is a CIFAR Senior Fellow.}}\\
D\'epartement d'informatique et de recherche op\'erationnelle\\
Universit\'e de Montr\'eal\\
Montr\'eal, QC H3C 3J7 \\
}

\nipsfinalcopy % Uncomment for camera-ready version

\begin{document}

\maketitle

\begin{abstract}
We propose a new framework for estimating generative models
via an adversarial process, in which we simultaneously train
two models: a generative model $G$ that captures the data
distribution, and a discriminative model $D$ that estimates the
probability that a sample came from the training data rather than $G$.
The training procedure for $G$ is to maximize the probability of $D$ making
a mistake. This framework corresponds to a minimax two-player game. In the
space of arbitrary functions $G$ and $D$, a unique solution exists, with $G$
recovering the training data distribution and $D$ equal to $\frac{1}{2}$
everywhere. In the case where $G$ and $D$ are defined by multilayer perceptrons,
the entire system can be trained with backpropagation. There is no need for any Markov chains
or unrolled approximate inference networks during either training or generation
of samples. Experiments demonstrate the potential of the framework
through qualitative and quantitative evaluation of the generated
samples. 
\end{abstract}

\section{Introduction}

The promise of deep learning is to discover rich, hierarchical models~\citep{Bengio-2009-book} that represent
probability distributions over the kinds of data encountered in artificial intelligence applications,
 such as natural images, audio waveforms containing speech, and symbols in natural language corpora.
So far, the most striking successes in deep learning have involved discriminative models,
usually those that map a high-dimensional, rich sensory input to a class 
label~\citep{Hinton-et-al-2012,Krizhevsky-2012-small}. These
striking successes have primarily been based on the backpropagation and dropout algorithms, using
piecewise linear units
~\cite{Jarrett-ICCV2009,Glorot+al-AI-2011-small,Goodfellow_maxout_2013} which have a particularly well-behaved gradient
. Deep {\em generative} models have had less of an impact, due to the difficulty of approximating
many intractable probabilistic computations that arise in maximum likelihood estimation and
related strategies, and due to difficulty of leveraging the benefits of piecewise linear units in
the generative context. We propose a new generative model estimation procedure that sidesteps these difficulties.
\footnote{All code and hyperparameters available at \url{http://www.github.com/goodfeli/adversarial}}

In the proposed {\em adversarial nets} framework, the generative model is pitted against an
adversary: a discriminative model that learns to determine whether a sample is
from the model distribution or the data distribution. The generative model can be thought of
as analogous to a team of counterfeiters, trying to produce fake currency and use it without
detection, while the discriminative model is analogous to the police, trying to detect the
counterfeit currency. Competition in this game drives both teams to improve their methods until the counterfeits
are indistiguishable from the genuine articles.

This framework can yield specific training algorithms for many kinds of model and optimization
algorithm. In this article, we explore the special case when the generative model generates
samples by passing random noise through a multilayer perceptron, and the discriminative model
is also a multilayer perceptron. We refer to this special case as {\em adversarial nets}.
In this case, we can train both models using only
the highly successful backpropagation and dropout algorithms~\citep{Hinton-et-al-arxiv2012} 
and sample from the generative
model using only forward propagation. No approximate inference or Markov chains are necessary.

\section{Related work}
\label{sec:related}

%
% In principle, an SBN could be trained using AMF using an appropriate estimator of the gradient 
% (basically to ``back-prop'' through stochastic binary units), e.g., using the REINFORCE
% algorithm~\citep{Williams-1992} or other variants such as previously explored by~\citet{bengio2013estimating}.
% On the other hand, as we show here, a generative network similar to an SBN but with continuous latent variables
% can be trained using AMF. 

% I THINK WE SHOULD BE VERY CAREFUL NOT TO BE TOO CRITICAL OF ALL THE PREVIOUS APPROACHES.
% the way I would put it is that unsupervised learning of generative models is fundamentally hard,
% and that many different approaches have been explored in the past, yet without a breakthrough
% similar to what we have seen with supervised deep nets. What we offer here is the possibility
% of a completely different approach, which sidesteps issues that may hurt some (but not all)
% of the other approaches.

% YB: not necessary

An alternative to directed graphical models with latent variables are undirected graphical models
with latent variables, such as restricted Boltzmann machines (RBMs)~\citep{Smolensky86,Hinton06}, 
deep Boltzmann machines (DBMs)~\citep{Salakhutdinov+Hinton-2009-small} and their numerous variants.
The interactions within such models are represented as the product of
unnormalized potential functions, normalized by a global
summation/integration over all states of the random variables.
This quantity (the \textit{partition function}) and
its gradient are intractable for all but the most trivial
instances, although they can be estimated by Markov chain Monte
Carlo (MCMC) methods. Mixing poses a significant problem for learning
algorithms that rely on
MCMC~\citep{Bengio-et-al-ICML2013,Bengio-et-al-ICML2014}.
% many typical transition operators are very local, making it difficult
% to cross deserts of low probability that can exist between modes. 
% Tempering methods have been proposed to
% handle this
% problem~\citep{Neal94b,Desjardins+al-2010,Cho10IJCNN,Salakhutdinov-ICML2010},
% although rejected swaps and the requirement of maintaining several
% parallel Markov chains at different temperatures renders these methods
% computationally expensive in practice.

Deep belief networks (DBNs)~\citep{Hinton06} are hybrid models containing a single undirected layer and several directed layers.
While a fast approximate layer-wise training criterion exists, DBNs incur the computational
difficulties associated with both undirected and directed models.

Alternative criteria that do not approximate or bound
the log-likelihood have also been proposed, such as score matching~\citep{Hyvarinen-2005-small}
and noise-contrastive estimation (NCE)~\citep{Gutmann+Hyvarinen-2010-small}.
Both of these require the learned probability density to be analytically specified
up to a normalization constant. Note that in many interesting generative models with several
layers of latent variables (such as
DBNs and DBMs), it is not even possible to derive a tractable unnormalized probability density.
Some models such as denoising auto-encoders~\citep{VincentPLarochelleH2008-small} and contractive
autoencoders have learning rules very similar to score matching applied to RBMs.
In NCE, as in this work, a discriminative training criterion is
employed to fit a generative model. However, rather than fitting a separate discriminative
model, the generative model itself is used to discriminate generated data from samples
a fixed noise distribution. Because NCE uses a fixed noise distribution, learning slows
dramatically after the model has learned even an approximately correct distribution over
a small subset of the observed variables.

%Noise-contrastive estimation is similar in spirit to the AMF, in that one trains a classifier that
%distinguishes the generated data from another distribution but maximizes the
%probability that the classifier gives to training examples being classified
%as coming from the learned model. A major difference with the AMF is that
%the NCE framework requires the learned probabilistic model of the data
%to be formulated in terms of an unnormalized probability (from an
%observation $x$ to an unnormalized probability $p$). Instead,
%with the AMF, the learned model is formulated as a generative procedure,
%as a function from randomly generated noise $z$ to a generated sample $x$.
%In addition, with the AMF, one learns to distinguish the training data from
%the generated data (instead of from a broad prior distribution), which we suspect
%should yield faster statistical convergence. This relates, in spirit, 
%to the log-likelihood gradient (as realized e.g. in Boltzmann
%machines): increase the propensity of the generator to
%produce samples like those in the training set while decreasing its propensity
%to produce samples like those it currently produces. When these two forces
%reach an equilibrium, it means that the generated samples are indistinguishable
%from the training examples, according to the model.
%

Finally, some techniques do not involve defining a probability distribution explicitly,
but rather train a generative machine to draw samples from the desired distribution.
This approach has the advantage that such machines can be designed to be trained by
back-propagation. Prominent recent work in this area includes the generative stochastic network
(GSN) framework~\citep{Bengio-et-al-ICML2014}, which extends generalized denoising
auto-encoders~\citep{Bengio-et-al-NIPS2013}: both can be seen as defining a
parameterized Markov chain, i.e., one learns the parameters of a machine that
performs one step of a generative Markov chain.
Compared to GSNs, the adversarial nets framework does not require a Markov chain for
sampling. Because adversarial nets do not require feedback loops during generation,
they are better able to leverage piecewise linear
units~\cite{Jarrett-ICCV2009,Glorot+al-AI-2011-small,Goodfellow_maxout_2013},
which improve the performance of backpropagation but have problems with unbounded activation when used ina feedback loop.
%In the case of GSNs, the generative
%network is a recurrent network with noise injected in the input and hidden
%units, which outputs sampled values at each step of the recurrence (which
%corresponds to a step of a Markov chain). 
%The idea of parameterizing
%continuous latent variables as deterministic functions to allow the application
%of backpropagation learning
%%%%into them and integrate over the noise variables rather than over the latent variables 
%has been independently proposed by several
%authors~\citep{bengio2013estimating,Bengio+Laufer-arxiv-2013,Kingma-arxiv2013,Kingma+Welling-arxiv2014}.
%The idea of the change of variable to enable the computation of derivatives
%through statistical models actually dates back decades in the statistical
%literature~\citet{Price-1958,Bonnet-1964}.
More recent examples of training a generative machine by back-propagating into it
include recent work on auto-encoding variational Bayes~\citep{Kingma+Welling-ICLR2014} 
and stochastic backpropagation~\citep{Rezende-et-al-arxiv2014}.
%In both
%cases one maximizes a variational lower bound on the log-likelihood that is
%rewritten as two terms: one that is just reconstruction log-likelihood
%through a stochastic encoder (approximate inference) - decoder (generative model) pair, and
%one that regularizes the output of the approximate inference stochastic encoder so that
%its marginal distribution matches the generative prior on the latent variables (and the
%latter can also trained, to match the marginal of the encoder output).

%\vfill
\section{Adversarial nets}

The adversarial modeling framework is most straightforward to apply when the models are both
multilayer perceptrons. To learn the generator's distribution $p_g$ over data $\bm{x}$, we
define a prior on input noise variables $p_{\bm{z}}(\bm{z})$, then represent a
mapping to data space as $G(\bm{z}; \theta_g)$, where $G$ is a differentiable function
represented by a multilayer perceptron with parameters $\theta_g$. We also define a second
multilayer perceptron $D(\bm{x}; \theta_d)$ that outputs a single scalar. $D(\bm{x})$ represents
the probability that $\bm{x}$ came from the data rather than $p_g$. We train $D$ to maximize the
probability of assigning the correct label to both training examples and samples from $G$.
We simultaneously train $G$ to minimize $\log(1-D(G(\bm{z})))$:

In other words, $D$ and $G$ play the following two-player minimax game with value function $V(G, D)$: 

\begin{equation}
\label{eq:minimaxgame-definition}
\min_G \max_D V(D, G) = \mathbb{E}_{\bm{x} \sim p_{\text{data}}(\bm{x})}[\log D(\bm{x})] + \mathbb{E}_{\bm{z} \sim p_{\bm{z}}(\bm{z})}[\log (1 - D(G(\bm{z})))].
\end{equation}

In the next section, we present a theoretical analysis of adversarial nets,
essentially showing that the training criterion allows one to recover the data
generating distribution as $G$ and $D$ are given enough capacity, i.e., in the
non-parametric limit. See Figure~\ref{fig:intuition} for a less formal, more pedagogical
explanation of the approach.
In practice, we must implement the game using an iterative, numerical approach. Optimizing $D$ to completion in the
inner loop of training is computationally prohibitive,
and on finite datasets would result in overfitting. Instead, we alternate between $k$ steps
of optimizing $D$ and one step of optimizing $G$. This results in $D$ being maintained
near its optimal solution, so long as $G$ changes slowly enough. This strategy is analogous
to the way that SML/PCD~\citep{Younes1999,Tieleman08} training maintains samples from a Markov chain from one
learning step to the next in order to avoid burning in a Markov chain as part of the inner loop
of learning. The procedure is formally presented
in Algorithm~\ref{alg:AGF}.

In practice, equation~\ref{eq:minimaxgame-definition} may not provide sufficient gradient for $G$ to learn
well. Early in learning, when $G$ is poor, $D$ can reject samples with high confidence because they are
clearly different from the training data. In this case, $\log ( 1- D(G(\bm{z})))$ saturates. Rather than
training $G$ to minimize $\log (1 - D(G(\bm{z})))$ we can train $G$ to maximize $\log D(G(\bm{z}))$.
This objective function results in the same fixed point of the dynamics of $G$ and $D$ but provides much
stronger gradients early in learning.

%explains the intuition
%behind these asymptotic consistency results. When $D$ tracks its optimum,
%it classifies $x$'s according to Bayes rule, i.e., 
%$D^*(x) = \frac{P_X}{P_X(x) + P_G(x)}$, and the gradient of $D(G(z))$ on $G(z)$ pushes
%probability mass in the direction of increasing value of $D$, i.e., towards regions 
%where $P_X(x)>P_G(x)$. This necessarily pulls probability mass away from
%regions where $P_X(x)<P_G(x)$. This happens because on the border of regions
%where $P_X(x)>P_G(x)$ the derivative of $D^*$ must point inside, and vice-versa.
%Hence the gradient pushes $G$ towards allocating more mass where it did not put enough,
%and vice-versa. A more formal proof of convergence is presented in Section~\ref{},
%also showing that the criterion is consistent, so long as $D$ is allowed to track
%its optimum, i.e., $G$ is making small changes between optimizations of $D$.

\begin{figure}[h]
\begin{tabular}{m{3cm}m{3cm}m{3cm}m{0.1cm}m{3cm}}
    \includegraphics[width=3cm, height=4cm]{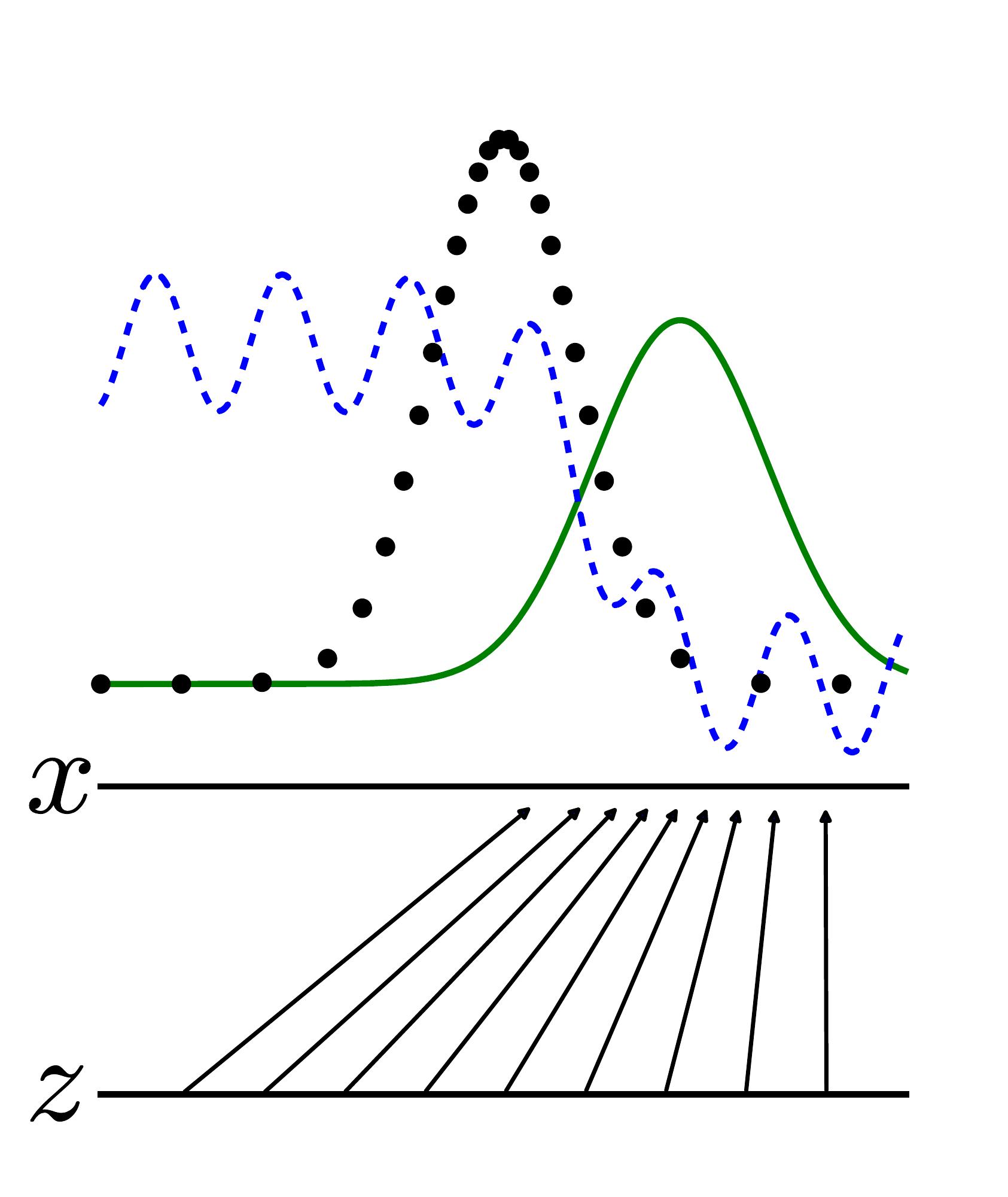} 
    &  
    \includegraphics[width=3cm, height=4cm]{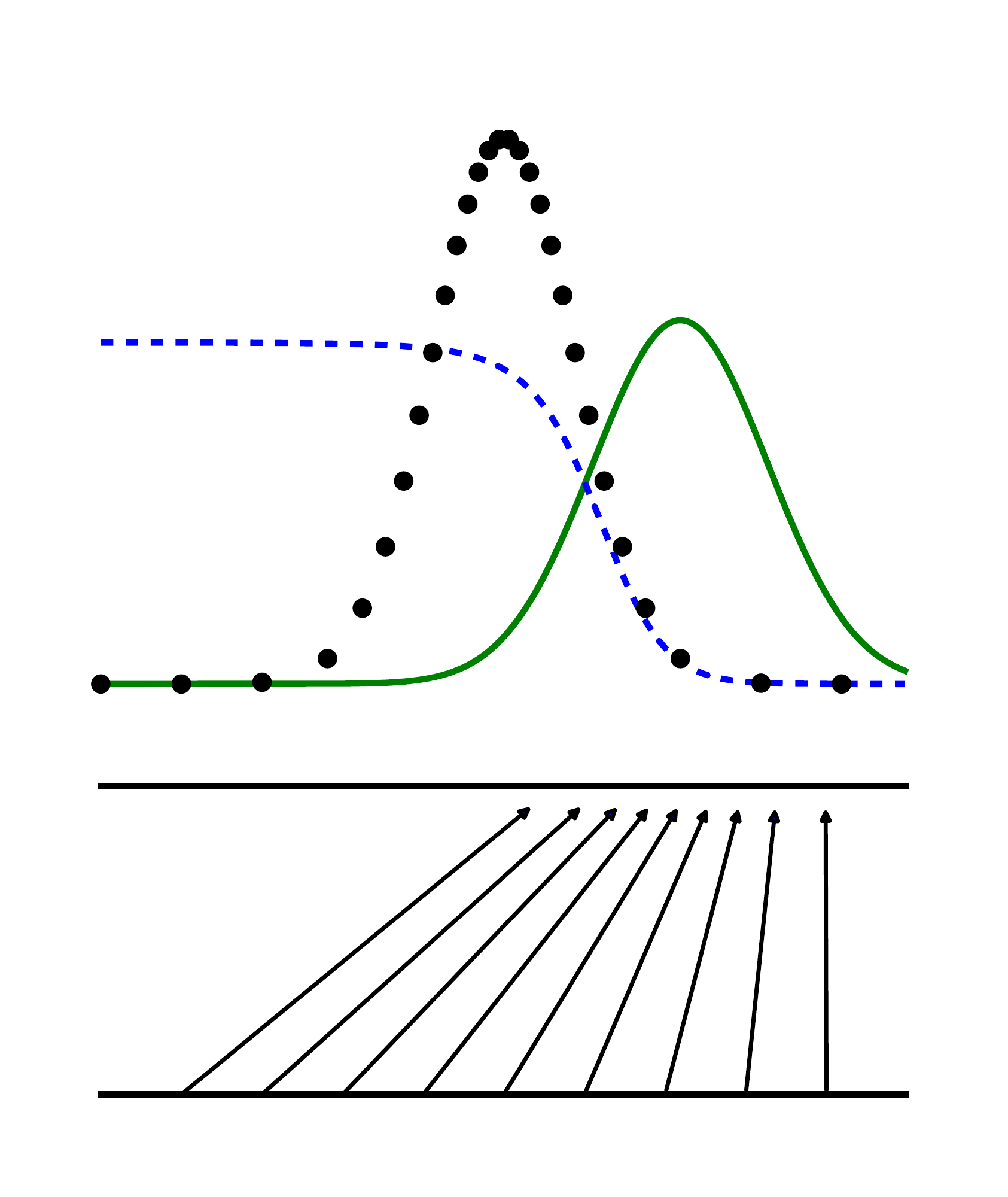}
    & 
    \includegraphics[width=3cm, height=4cm]{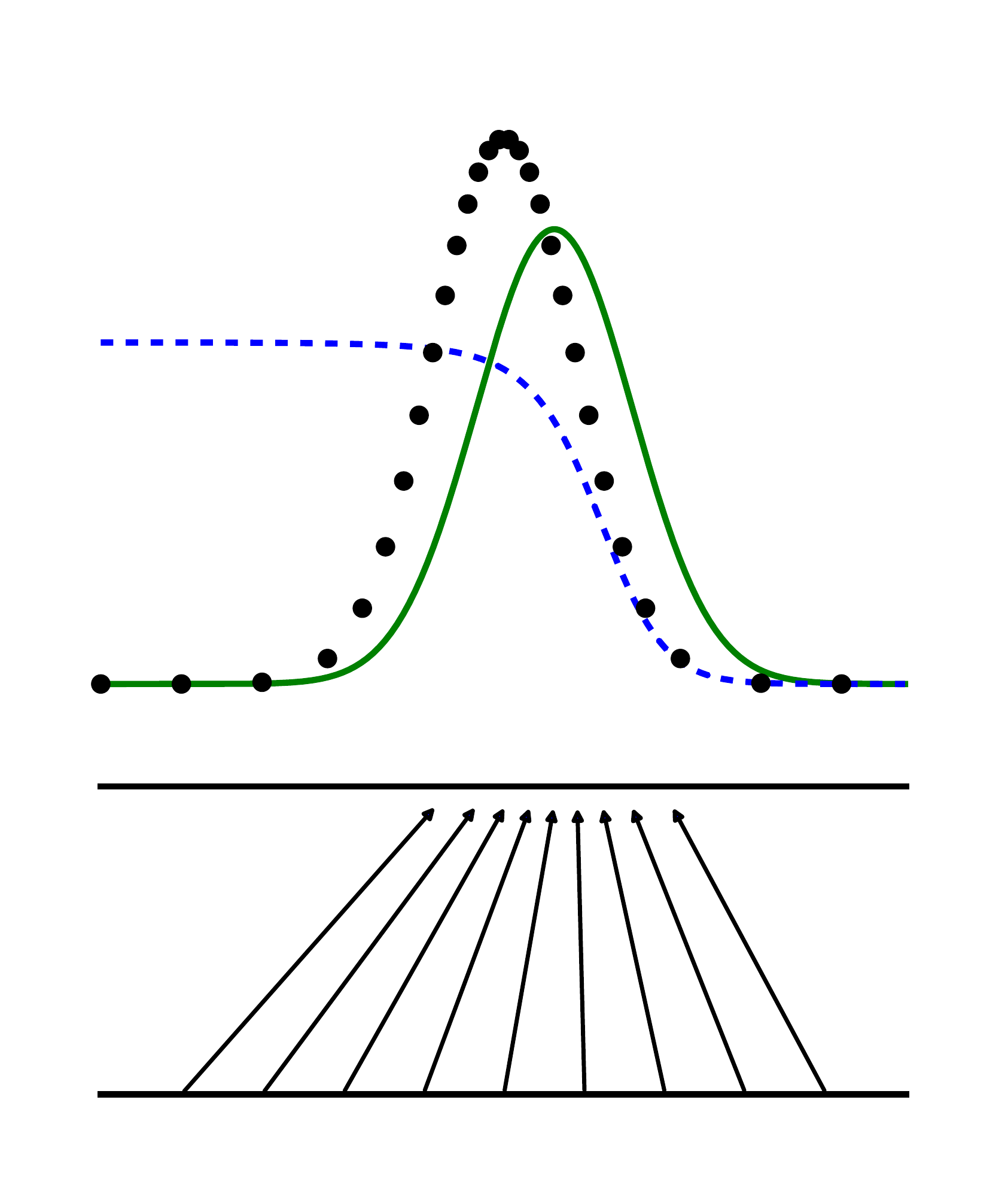}
    &
    \dots
    &
    \includegraphics[width=3cm, height=4cm]{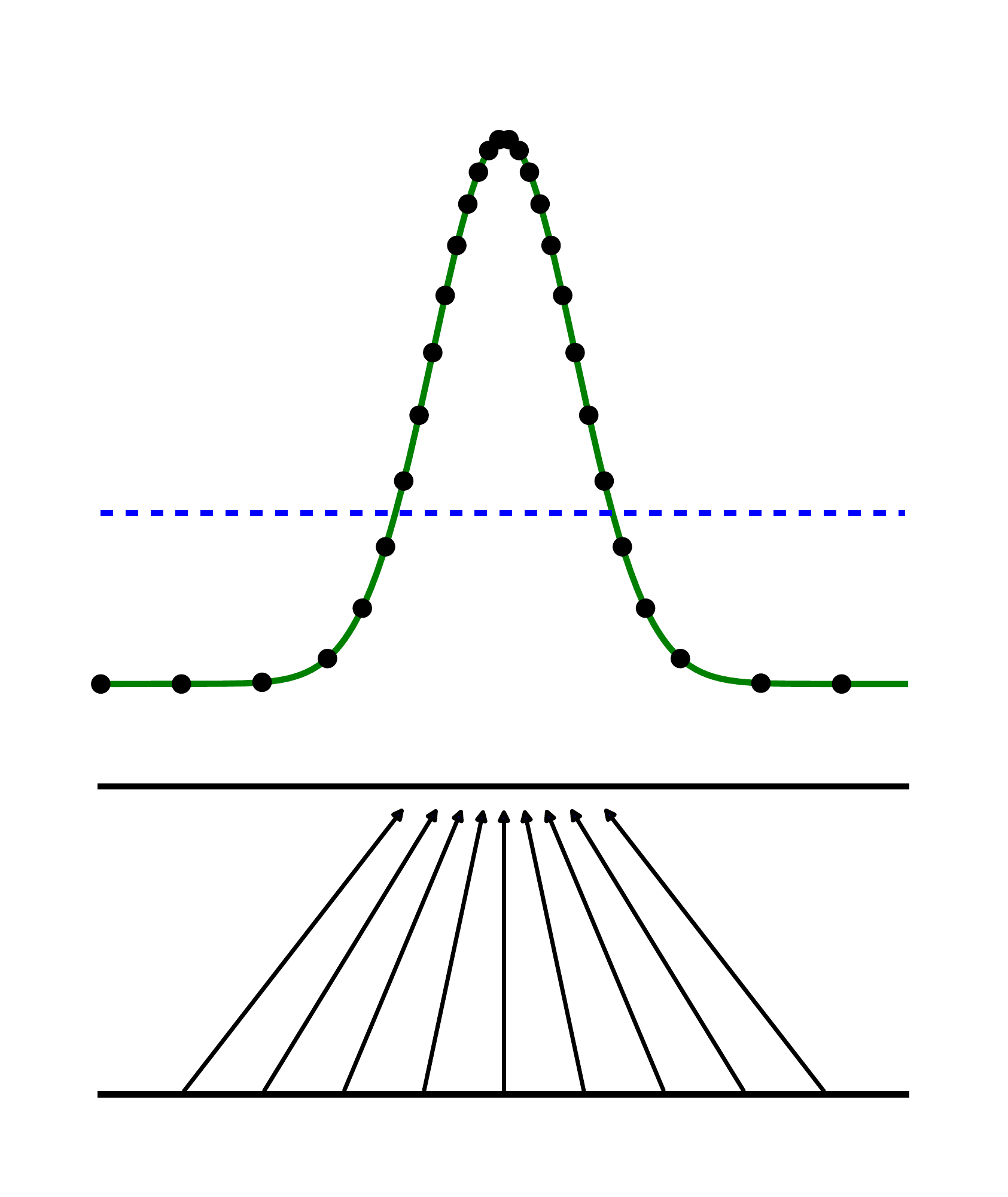}
    \\
    \centering (a)
    & 
    \centering (b) 
    & 
    \centering (c) 
    &  
    & 
    \centering (d)
\end{tabular}

\caption{\small
Generative adversarial nets are trained by simultaneously updating the \textbf{d}iscriminative 
distribution ($D$, blue, dashed line) so that it discriminates between samples from
the data generating distribution (black, dotted line)
$p_{\bm{x}}$ 
from those of the \textbf{g}enerative distribution $p_g$ (G) (green, solid line).
The lower horizontal line is the domain from which $\bm{z}$ is sampled, in this case uniformly.
The horizontal line above is part of the domain of $\bm{x}$. The upward arrows show
how the mapping $\bm{x}=G(\bm{z})$ imposes the non-uniform distribution $p_g$ on transformed samples.
$G$ contracts in regions of high density and expands in regions of low density of $p_g$.
(a) Consider an adversarial pair near convergence: $p_g$ is similar to $p_\text{data}$ and
$D$ is a partially accurate classifier.
(b) In the inner loop of the algorithm $D$ is trained to discriminate samples from data,
converging to
$D^*(\bm{x}) = 
\frac{
    p_\text{data}(\bm{x})
    }{
        p_\text{data}(\bm{x}) + p_g(\bm{x})}
$. 
(c) After an update to $G$, gradient of $D$ has guided $G(\bm{z})$ to flow to regions that are more likely
to be classified as data.
(d) After several steps of training, if $G$ and $D$ have enough capacity, they will reach a point at which
both cannot improve because $p_g = p_\text{data}$.
The discriminator is unable to differentiate between the two 
distributions, i.e. $D(\bm{x}) = \frac{1}{2}$.
}
\label{fig:intuition}
\end{figure}

\begin{algorithm}[ht]
\caption{\small Minibatch stochastic gradient descent training of generative adversarial nets.
The number of steps to apply to the discriminator, $k$, is a hyperparameter. We used $k=1$, the
least expensive option, in our experiments.
}
\begin{algorithmic}
\label{alg:AGF}
\FOR{number of training iterations}
  \FOR{$k$ steps}
    \STATE{$\bullet$ Sample minibatch of $m$ noise samples $\{ \bm{z}^{(1)}, \dots, \bm{z}^{(m)} \}$ from noise prior $p_g(\bm{z})$.}
    \STATE{$\bullet$ Sample minibatch of $m$ examples $\{ \bm{x}^{(1)}, \dots, \bm{x}^{(m)} \}$ from data generating distribution $p_\text{data}(\bm{x})$.}
    \STATE{$\bullet$ Update the discriminator by ascending its stochastic gradient:
        \[
            \nabla_{\theta_d} \frac{1}{m} \sum_{i=1}^m \left[
            \log D\left(\bm{x}^{(i)}\right)
            + \log \left(1-D\left(G\left(\bm{z}^{(i)}\right)\right)\right)
            \right].
        \]}
    %parameters $\theta_d$ of discriminator $D$
   %in the direction of the stochastic gradient of the binomial cross-entropy
   %for $D$ predicting whether its argument comes from $p_\text{data}(\bm{x})$ (target = 1, input = $\bm{x}$) or
   %$P_g$ (target = 0, input = $G(\bm{z})$), i.e., towards minimizing
   % \mbox{$-\log D(\bm{x}) - \log(1 - D(G(\bm{z})))$}.}
   \ENDFOR
  \STATE{$\bullet$ Sample minibatch of $m$ noise samples $\{ \bm{z}^{(1)}, \dots, \bm{z}^{(m)} \}$ from noise prior $p_g(\bm{z})$.}
    \STATE{$\bullet$ Update the generator by descending its stochastic gradient:
        \[
            \nabla_{\theta_g} \frac{1}{m} \sum_{i=1}^m
            \log \left(1-D\left(G\left(\bm{z}^{(i)}\right)\right)\right)
            .
        \]}
  \ENDFOR
  \\The gradient-based updates can use any standard gradient-based learning rule. We used momentum in our experiments.
\end{algorithmic}
\end{algorithm}

\section{Theoretical Results}
\label{sec:theory}

The generator $G$ implicitly defines a probability distribution $p_g$ as
the distribution of the samples $G(\bm{z})$ obtained when $\bm{z} \sim
p_{\bm{z}}$. Therefore, we would like Algorithm~\ref{alg:AGF} to converge to a
good estimator of $p_\text{data}$, if given enough capacity and training time. The
results of this section are done in a non-parametric setting, e.g. we represent a
model with infinite capacity by studying convergence in the space of probability
density functions.

% Note: IG moved the minimax game definition above

We will show in section~\ref{sec:global-optimality} that this minimax game
has a global optimum for $p_g = p_\text{data}$. We will then show in
section~\ref{sec:convergence-of-algorithm} that Algorithm~\ref{alg:AGF}
optimizes Eq~\ref{eq:minimaxgame-definition}, thus obtaining the desired
result.

\subsection{Global Optimality of $p_g=p_\text{data}$}
\label{sec:global-optimality}

We first consider the optimal discriminator $D$ for any given generator $G$.

\begin{proposition}
For $G$ fixed, the optimal discriminator $D$ is
\begin{equation}
\label{eq:optimal-D}
   D^*_G(\bm{x}) = \frac{p_\text{data}(\bm{x})}{p_\text{data}(\bm{x}) + p_g(\bm{x})}
\end{equation}
\end{proposition}

\begin{proof}
The training criterion for the discriminator D, given any generator $G$, is to maximize the quantity $V(G, D)$
\begin{align}
V(G, D) =& \int_{\bm{x}} p_\text{data}(\bm{x}) \log(D(\bm{x})) dx + \int_z p_{\bm{z}}(\bm{z}) \log(1 - D(g(\bm{z}))) dz \nonumber \\
\label{eq:cost-for-D}
		=& \int_{\bm{x}} p_\text{data}(\bm{x}) \log(D(\bm{x})) + p_g(\bm{x}) \log(1 - D(\bm{x})) dx
\end{align}

For any $(a,b) \in \mathbb{R}^2 \setminus \{0, 0\} $, the function $y \rightarrow a \log (y) + b \log (1-y)$ achieves its maximum in $[0,1]$ at $\frac{a}{a+b}$. The discriminator does not need to be defined outside of $Supp(p_\text{data}) \cup Supp(p_g)$, concluding the proof.
\end{proof}

Note that the training objective for $D$ can be interpreted as maximizing the log-likelihood for estimating the conditional probability $P(Y=y|\bm{x})$, where $Y$ indicates whether $\bm{x}$ comes from $p_\text{data}$ (with $y=1$) or from $p_g$ (with $y=0$). The minimax game in Eq.~\ref{eq:minimaxgame-definition} can now be reformulated as:

\begin{align}
\label{eq:G-criterion}
 C(G) =& \max_D V(G,D) \nonumber \\
      =& \mathbb{E}_{\bm{x} \sim p_\text{data}}[\log D^*_G(\bm{x})] + \mathbb{E}_{\bm{z} \sim p_{\bm{z}}}[\log (1 - D^*_G(G(\bm{z})))] \\
      =& \mathbb{E}_{\bm{x} \sim p_\text{data}}[\log D^*_G(\bm{x})] + \mathbb{E}_{\bm{x} \sim p_g}[\log (1 - D^*_G(\bm{x}))] \nonumber \\
      =& \mathbb{E}_{\bm{x} \sim p_\text{data}}\left[\log \frac{p_\text{data}(\bm{x})}{P_\text{data}(\bm{x}) + p_g(\bm{x})}\right] + 
         \mathbb{E}_{\bm{x} \sim p_g}\left[\log \frac{p_g(\bm{x})}{p_\text{data}(\bm{x}) + p_g(\bm{x})}\right] \nonumber  
\end{align}

\begin{theorem}
\label{thm:global-optimality}
The global minimum of the virtual training criterion $C(G)$ is achieved if and only if $p_g=p_\text{data}$.
At that point, $C(G)$ achieves the value $-\log 4$.
\end{theorem}

\begin{proof}
For $p_g=p_\text{data}$, $D^*_G(\bm{x})=\frac{1}{2}$, (consider Eq.~\ref{eq:optimal-D}). Hence, by inspecting Eq.~\ref{eq:G-criterion} at $D^*_G(\bm{x})=\frac{1}{2}$, we find $C(G)=\log\frac{1}{2} + \log\frac{1}{2}= - \log 4$. To see that this is the best possible value of $C(G)$, reached only for $p_g = p_\text{data}$, observe that
\[
 \mathbb{E}_{\bm{x} \sim p_\text{data}} \left[ - \log 2 \right] + \mathbb{E}_{\bm{x} \sim p_g} \left[ - \log 2 \right] = -\log 4
\]
%\begin{equation}
% - \log (4) = \mathbb{E}_{\bm{x} \sim p_\text{data}} \left[ - \log 2 \right] + \mathbb{E}_{\bm{x} \sim p_g} \left[ - \log 2 \right]
%\end{equation}
and that by subtracting this expression from $C(G) = V(D_G^*,G)$, we obtain:

\begin{equation}
%\label{eq:sym-KL}
C(G) = -\log(4) + KL \left(p_\text{data} \left \| \frac{p_\text{data} + p_g}{2} \right. \right) + KL \left(p_g \left \| \frac{p_\text{data} + p_g}{2} \right. \right)
\end{equation}

where KL is the Kullback--Leibler divergence. We recognize in the previous expression the Jensen--Shannon divergence between the model's distribution and the data generating process:

\begin{equation}
%\label{eq:sym-KL}
C(G) = - \log(4) + 2 \cdot JSD \left(p_\text{data} \left \| p_g \right. \right)
\end{equation}

Since the Jensen--Shannon divergence between two distributions is always non-negative and zero only
when they are equal, we have shown that $ C^* = - \log (4)$ is the global minimum of $C(G)$ and
that the only solution is $p_g=p_\text{data}$, i.e., the generative model perfectly replicating the data generating process.
\end{proof}

\subsection{Convergence of Algorithm~\ref{alg:AGF}}
\label{sec:convergence-of-algorithm}

\begin{proposition}
If $G$ and $D$ have enough capacity, and at each step of Algorithm~\ref{alg:AGF}, the discriminator is allowed to reach its optimum given $G$, and $p_g$ is updated so as to improve the criterion $$\mathbb{E}_{\bm{x} \sim p_\text{data}}[\log D^*_G(\bm{x})] + \mathbb{E}_{\bm{x} \sim p_g}[\log (1 - D^*_G(\bm{x}))]$$ then $p_g$ converges to $p_\text{data}$
\end{proposition}

\begin{proof}
Consider $V(G,D)=U(p_g,D)$ as a function of $p_g$ as done in the above criterion. Note that $U(p_g,D)$ is convex in $p_g$. The subderivatives of a supremum of convex functions include the derivative of the function at the point where the maximum is attained. In other words, if $f(x) = \sup_{\alpha \in \cal{A}} f_\alpha(x)$ and $f_\alpha(x)$ is convex in $x$ for every $\alpha$, then $\partial f_\beta(x) \in \partial f$ if $\beta = \arg \sup_{\alpha \in \cal{A}} f_\alpha(x)$. This is equivalent to computing a gradient descent update for $p_g$ at the optimal $D$ given the corresponding $G$. $\sup_D U(p_g,D)$ is convex in $p_g$ with a unique global optima as proven in Thm~\ref{thm:global-optimality}, therefore with sufficiently small updates of $p_g$, $p_g$ converges to $p_x$, concluding the proof.
% It is a known result of convex
% optimization then that the gradient updates (performed in
% Algorithm~\ref{alg:AGF}) are in the subderivatives of $U(p_g,D)$ with
% respect to $p_g$. To decrease the maximum over the set of $D$'s, it is enough to
% make a small enough step of $G$ according to the subgradients, at the given $D$ which
% maximized $V(G,D)$, i.e., since decreasing $V(G,D)$ decreases $U(p_g,D)$. 
% With sufficiently small steps of $G$, Algorithm~\ref{alg:AGF} thus reaches the 
% global optimum $p_g = p_\text{data}$ of the minimax game.
\end{proof}

In practice, adversarial nets represent a limited family of $p_g$ distributions via the function $G(\bm{z}; \theta_g)$,
and we optimize $\theta_g$ rather than $p_g$ itself. Using a multilayer perceptron to define $G$ introduces multiple
critical points in parameter space.
However, the excellent performance of multilayer
perceptrons in practice suggests that they are a reasonable model to use despite their lack of theoretical guarantees.

\section{Experiments}

We trained adversarial nets an a range of datasets including MNIST\citep{LeCun+98}, the
Toronto Face Database (TFD)~\cite{Susskind2010}, and CIFAR-10~\citep{KrizhevskyHinton2009}.
The generator nets used a mixture of rectifier linear
activations~\cite{Jarrett-ICCV2009,Glorot+al-AI-2011-small} and sigmoid
activations, while the discriminator net used maxout~\cite{Goodfellow_maxout_2013} activations.
Dropout~\cite{Hinton-et-al-arxiv2012} was applied in training the
discriminator net. While our theoretical framework permits the use of dropout and other noise
at intermediate layers of the generator, we used noise as the input to only the bottommost layer
of the generator network.

We estimate probability of the test set data under $p_g$ by fitting a Gaussian Parzen window to the samples
generated with $G$ and reporting the log-likelihood under this distribution.
The $\sigma$ parameter of the Gaussians
was obtained by cross validation on the validation set. This procedure was introduced in  \citet{Breuleux+Bengio-2011}
and used for various generative models for which the exact likelihood is
not
tractable~\citep{Rifai-icml2012,Bengio-et-al-ICML2013,Bengio-et-al-ICML2014}. Results
are reported in Table \ref{table:parzen}. This method of estimating the likelihood has somewhat high variance
and does not perform well in high dimensional spaces but it is the best method available to our knowledge.
Advances in generative models that can sample but not estimate likelihood directly motivate further research
into how to evaluate such models.

\begin{table}
\centering
\begin{tabular}{c|c|c}
Model & MNIST & TFD \\
\hline
DBN~\citep{Bengio-et-al-ICML2013} & $138 \pm 2$ & $1909 \pm 66$ \\
Stacked CAE~\citep{Bengio-et-al-ICML2013} & $121 \pm 1.6$ & $\mathbf{2110 \pm 50}$ \\
Deep GSN~\citep{Bengio-et-al-ICML-2014} & $214 \pm 1.1$ & $1890 \pm 29$ \\
Adversarial nets & $\mathbf{225 \pm 2}$ & $\mathbf{2057 \pm 26}$
\end{tabular}
\caption{\small
Parzen window-based log-likelihood estimates. The reported numbers on MNIST are the mean log-likelihood of samples on test set,
with the standard error of the mean computed across examples.
On TFD, we computed the standard error across folds of the dataset, with a different $\sigma$ chosen using the validation set
of each fold.
On TFD, $\sigma$ was cross validated on each fold and mean log-likelihood on each fold were computed.
For MNIST we compare against other models of the real-valued (rather than binary) version of dataset.}
\label{table:parzen}
\end{table}

In Figures \ref{fig:visuals1} and \ref{fig:visuals3} we
show samples drawn from the generator net after training. 
While we make no claim that these samples are better than
samples generated by existing methods, we believe that these samples are at
least competitive with the better generative models in the literature and
highlight the potential of the adversarial framework.

\begin{figure}[h]
\centering
\begin{tabular}{cc}
 \includegraphics[width=2.6in]{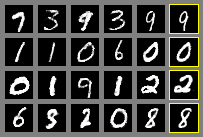} &
 \includegraphics[width=2.6in]{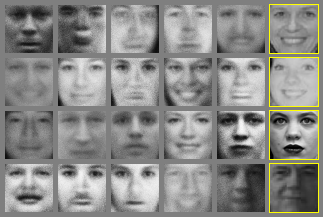} \\
a) & b) \\
 \includegraphics[width=2.6in]{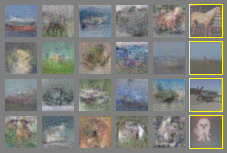} &
 \includegraphics[width=2.6in]{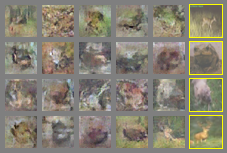} \\
 c) & d)
\end{tabular}
\caption{
\small Visualization of samples from the model. Rightmost column shows the
nearest training example of the neighboring sample, in order to demonstrate
that the model has not memorized the training set. Samples are fair random
draws, not cherry-picked. Unlike most other visualizations of deep generative
models, these images show actual samples from the model distributions, not
conditional means given samples of hidden units. Moreover, these samples are
uncorrelated because the sampling process does not depend on Markov chain mixing.
a) MNIST b) TFD c) CIFAR-10 (fully connected model)
d) CIFAR-10 (convolutional discriminator and ``deconvolutional'' generator)
}
\label{fig:visuals1}
\end{figure}

%\begin{figure}[h]
%\centering
%    \includegraphics[width=10cm]{trans.png}
%\end{figure}

\iffalse
\begin{figure}[h]
\centering
    \includegraphics[width=12cm]{2dmanifold.png} 
    \caption{\small The manifold of MNIST digits learned by a small model
      with $\bm{z} \in \mathbf{R}^2$. }
\label{fig:visuals2}
\end{figure}
\fi

\begin{figure}[h]
\centering
    \includegraphics[width=6cm]{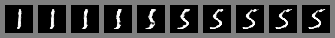} 
    \includegraphics[width=6cm]{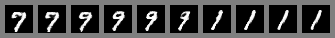} 
    \caption{\small Digits obtained by linearly interpolating between
      coordinates in $\bm{z}$ space of the full model.}
\label{fig:visuals3}
\end{figure}

\iffalse
\begin{figure}[h]
\centering
\begin{minipage}{.5\textwidth}
    \centering
    \includegraphics[width=7.5cm]{failure.pdf} 
    \caption{\small Failure case of discriminator}
    \label{fig:failure}
\end{minipage}%
\begin{minipage}{.5\textwidth}
    \centering
    \includegraphics[width=6cm]{helvetica.png} 
    \caption{\small TODO Training only G for a fixed D}
    \label{fig:helvetica}
\end{minipage}
\end{figure}
\fi

\newcolumntype{L}[1]{>{\raggedright}m{#1}}
\begin{table}[ht]
\small
\begin{tabular}{c|L{1in}|L{1in}|L{1in}|L{1in}}
& Deep directed graphical models & Deep undirected graphical models & Generative autoencoders & Adversarial models \tabularnewline
\hline
Training & Inference needed during training.
         & Inference needed during training. MCMC needed to approximate partition function gradient.
         & Enforced tradeoff between mixing and power of reconstruction generation
         & Synchronizing the discriminator with the generator. Helvetica.\tabularnewline
\hline
Inference &  Learned approximate inference &  Variational inference &  MCMC-based inference &  Learned approximate inference \tabularnewline
\hline
Sampling & No difficulties &  Requires Markov chain & Requires Markov chain & No difficulties \tabularnewline
\hline
Evaluating $p(x)$ & {  Intractable, may be approximated with AIS} &  Intractable, may be approximated with AIS  &
Not explicitly represented, may be approximated with Parzen density estimation &  Not explicitly represented, may be approximated with Parzen density estimation\tabularnewline
\hline
Model design &  Nearly all models incur extreme difficulty &  Careful design needed to ensure multiple properties
 &  Any differentiable function is theoretically permitted &  Any differentiable function is theoretically permitted
\end{tabular}
\caption{\small Challenges in generative modeling: a summary of the difficulties encountered by different approaches to deep generative modeling for each of
the major operations involving a model.}
\label{table:comparison}
\normalsize
\end{table}
\section{Advantages and disadvantages}

This new framework comes with advantages and disadvantages relative to previous
modeling frameworks. The disadvantages are primarily that there is no explicit
representation of $p_g(\bm{x})$, and that $D$ must be synchronized well with $G$
during training (in particular, $G$ must not be trained too much without updating
$D$, in order to avoid ``the Helvetica scenario'' in which $G$ collapses too many
values of $\mathbf{z}$ to the same value of $\mathbf{x}$ to have enough diversity
to model $p_\text{data}$),
much as the negative chains of a Boltzmann machine must be
kept up to date between learning steps. The advantages are that Markov chains
are never needed, only backprop is used to obtain gradients, 
no inference is needed during learning, and a wide variety of
functions can be incorporated into the model.
Table~\ref{table:comparison} summarizes the comparison of generative adversarial nets with other 
generative modeling approaches.

The aforementioned advantages are primarily computational. Adversarial models may
also gain some statistical advantage from the generator network not being updated
directly with data examples, but only with gradients flowing through the discriminator.
This means that components of the input are not copied directly into the generator's
parameters. Another advantage of adversarial networks is that they can represent
very sharp, even degenerate distributions, while methods based on Markov chains require
that the distribution be somewhat blurry in order for the chains to be able to mix
between modes.

\section{Conclusions and future work}

This framework
admits many straightforward extensions:
\begin{enumerate}[noitemsep,leftmargin=0.4cm]
\item A {\em conditional generative} model $p(\bm{x} \mid \bm{c})$ can be obtained by adding $\bm{c}$
as input to both $G$ and $D$.
\item {\em Learned approximate inference} can be performed by training an auxiliary network to predict
    $\bm{z}$ given $\bm{x}$. This is similar to the inference net trained by the wake-sleep algorithm ~\citep{Hinton95}
but with the advantage that the inference net may be trained for a fixed generator net after the generator
net has finished training.
\item One can approximately model all conditionals $p(\bm{x}_S \mid \bm{x}_{\not S})$ where $S$ is a subset of
the indices of $\bm{x}$
by training a family of conditional models that share parameters. Essentially, one can use adversarial nets to implement
 a stochastic extension of the deterministic MP-DBM \cite{Goodfellow-et-al-NIPS2013-small}.
\item {\em Semi-supervised learning}: features from the discriminator or inference net could improve performance
of classifiers when limited labeled data is available.
\item {\em Efficiency improvements:} training could be accelerated greatly by divising better methods for
coordinating $G$ and $D$ or determining better distributions to sample $\mathbf{z}$ from during training.
\end{enumerate}
This paper has demonstrated the viability of the adversarial modeling framework, suggesting that these research
directions could prove useful.

\subsubsection*{Acknowledgments}
We would like to acknowledge Patrice Marcotte, Olivier Delalleau, Kyunghyun Cho, Guillaume Alain and Jason Yosinski for helpful discussions.
Yann Dauphin shared his Parzen window evaluation code with us.
We would like to thank the developers of
Pylearn2~\citep{pylearn2_arxiv_2013}
and
Theano~\citep{bergstra+al:2010-scipy,Bastien-Theano-2012}, particularly
Fr\'ed\'eric Bastien who rushed a Theano feature specifically to benefit this project.
Arnaud Bergeron provided much-needed support with \LaTeX\ typesetting.
We would also like to thank CIFAR, and Canada Research Chairs for funding, and Compute Canada, and Calcul Qu\'ebec
for providing computational resources. Ian Goodfellow is supported by the 2013 Google Fellowship in Deep
Learning. Finally, we would like to thank Les Trois Brasseurs for stimulating our creativity.

{\small
\bibliography{strings,strings-shorter,ml,aigaion-shorter}
\bibliographystyle{natbib}}

\end{document}